\def\year{2019}\relax
\documentclass[letterpaper]{article}
\usepackage{aaai19}
\usepackage{times}
\usepackage{helvet}
\usepackage{courier}
\usepackage{amsfonts}
\usepackage{amsmath}
\usepackage{amsthm}
\usepackage{graphicx}
\usepackage{enumerate}
\usepackage{etex,etoolbox}
\usepackage{bm}
\usepackage{multirow}
\usepackage{amsthm}
\usepackage{amsmath}
\usepackage{mathrsfs}
\usepackage{bbm}
\usepackage{amssymb}

\newtheorem{theorem}{Theorem}
\newtheorem{proposition}[theorem]{Proposition}

\newtheorem{lem}[theorem]{Lemma}

\def\mdash{\hbox{-}}

\usepackage{booktabs}

\begin{document}

\title{Constraint-based Sequential Pattern Mining with Decision Diagrams}
\author{Amin Hosseininasab \\
Tepper School of Business,\\ Carnegie Mellon University, USA \\
aminh@andrew.cmu.edu
\And
Willem-Jan van Hoeve\\
Tepper School of Business,\\ Carnegie Mellon University, USA\\
vanhoeve@andrew.cmu.edu
\And
Andre A. Cire\\
Dept. of Management,\\
University of Toronto Scarborough, Canada\\
andre.cire@rotman.utoronto.ca}
\maketitle

\begin{abstract}
Constrained sequential pattern mining aims at identifying frequent patterns on a sequential database of items while observing constraints defined over the item attributes. We introduce novel techniques for constraint-based sequential pattern mining that rely on a multi-valued decision diagram representation of the database. Specifically, our representation can accommodate multiple item attributes and various constraint types, including a number of non-monotone constraints. To evaluate the applicability of our approach, we develop an MDD-based prefix-projection algorithm
and compare its performance against a typical generate-and-check variant, as well as a state-of-the-art constraint-based sequential pattern mining algorithm. Results show that our approach is competitive with or superior to these other methods in terms of scalability and efficiency.
\end{abstract}

\section{Introduction}
Sequential Pattern Mining (SPM) is a fundamental data mining task with a large array of applications in marketing, health care, finance, and bioinformatics, to name a few. Frequent patterns are used, e.g., to extract knowledge from data within decision support tools, to develop novel association rules, and to design more effective recommender systems.  We refer the reader to \cite{fournier2017survey} for a recent and thorough review of SPM and its applications.

In practice, mining the entire set of frequent patterns in a database is not of interest, as the resulting number of items is typically large and may provide no significant insight to the user. It is hence desirable to restrict the mining algorithm search to smaller subsets of patterns that satisfy problem-specific constraints. For example, in online retail click-stream analysis, we may seek frequent browsing patterns from sessions where users spend at least a minimum amount of time on certain items that have specific price ranges. Such constraints limit the output of SPM and are much more effective in knowledge discovery, as compared to an arbitrary large set of frequent click-streams.

A na\"ive approach to impose constraints in SPM is to first collect all unconstrained frequent patterns, and then to apply a post-processing step to retain patterns that satisfy the desired constraints. This approach, however, may be expensive in terms of memory requirements and computational time, in particular when the resulting subset of constrained patterns is small in comparison to the full unconstrained set. Constraint-based sequential pattern mining (CSPM) aims at providing more efficient methods by embedding constraint reasoning within existing mining algorithms~\cite{pei2007constraint,negrevergne2015constraint}. Nonetheless, while certain constraint types are relatively easy to incorporate in a mining algorithm, others of practical use are still challenging to handle in a general and effective way. This is particularly the case of non-monotone constraints representing, e.g., sums and averages of attributes.

\textit{Contributions.} In this paper, we propose a novel representation of sequential database using a multi-valued decision diagram (MDD), a graphical model that compactly encodes the sequence of items and their attributes by leveraging symmetry. The MDD representation can be augmented with constraint-specific information, so that constraint satisfaction is either guaranteed or enforced during the mining algorithm. Finally, as a proof of concept, we implement a general prefix-projection algorithm equipped with an MDD to enforce several constraint types, including complex constraints such as average
(``avg'') and median (``md''). To the best of our knowledge, this paper is the first to consider the ``sum,'' ``avg,'' and ``md'' constraints with arbitrary item-attribute association within the pattern mining algorithm. Lastly, we provide an experimental comparison on real-world benchmark databases, and show that our approach is competitive with or superior to a state-of-the-art CSPM algorithm.

\section{Related work}

Research in CSPM has primarily focused on exploiting special properties of constraints, such as monotonicity or anti-monotonicity, to guarantee the feasibility of pattern extensions in the mining algorithm \cite{garofalakis1999spirit,zaki2000sequence,lin2005efficient,bonchi2005pushing,chen2006constraint,pei2007constraint,nijssen2014constraint,mallick2014constraint,aoga2017mining}.  Constraint types that do not possess such properties remain a challenge for CSPM algorithms, although some of these have been successfully incorporated in more general item-set mining on databases where events have no specific order~\cite{soulet2005exploiting,bistarelli2007soft,bonchi2007extending,le2009optimal,leung2012constrained}, as well as in CSPM when items and attributes are interchangeable ~\cite{pei2007constraint}. 

Recently, constraint programming (CP) has emerged as a successful tool for CSPM \cite{negrevergne2015constraint,kemmar2016global,kemmar2017prefix,aoga2017mining,guns2017miningzinc}. CP search techniques, albeit general, can potentially be more efficient when compared to specialized CSPM algorithms. Nonetheless, they still rely on constraint-specific properties to effectively prune undesired patterns.For example, \cite{aoga2017mining} show how to effectively implement a number of prefix anti-monotone constraints into CP, but indicate that post-processing is still required to handle monotone constraints such as the minimum span. 

Graphical representations of a database have been shown to be effective in item-set mining \cite{han2004mining,pyun2014efficient,borah2018fp} and SPM \cite{masseglia2009efficient}. Previous works
have also applied binary decision diagrams as a database modeling tool \cite{loekito2006fast,loekito2007zero,loekito2010binary,cambazard2010knowledge}, which are effective when the sequences of the database are similar, but typically do not scale otherwise. We show that our MDD representation retains its size regardless of the similarity between sequences, and provides a more concise representation in the context of SPM. 

\section{Problem definition}\label{sec:probdef}

We next formally describe the SPM problem and then discuss the handling of constraints.

\subsection{The SPM database and mining algorithm}

The SPM database consists of a set of events, which are modeled by a set of literals $I$ denoted by \textit{items}. Items $i \in I$ are associated with a set of \textit{attributes} $\mathbb{A}=\left\lbrace \mathcal{A}_1,...,\mathcal{A}_{|\mathbb{A}|} \right\rbrace$; for example, attributes can be price, quality, or time. A {\em sequence database} $\mathcal{SD}$ is defined as a collection of $N$ item sequences $\left\lbrace S_1,S_2,\dots,S_N\right\rbrace$, where all sequences are ordered with respect to the same attribute $\mathcal{A} \in \mathbb{A}$; e.g., occurrence in time. Table \ref{tab:data} illustrates an example $\mathcal{SD}$ with $N := 3$, $|I|:=3$, and $M:=\underset{n\in \{1,...,N\}}{\max}\left\lbrace |S_n| \right\rbrace=3$, where items $i \in I$ are associated with time and price attributes.

\begin{table} 
\centering
\caption{Example $\mathcal{SD}$, with attributes of time and price.}
\label{tab:data}
\begin{tabular}{c l}\toprule
$S_{\rm ID}$&Sequence: $\left\lbrace (item, time, price) \right\rbrace $\\ \midrule
1&$\langle (B,1,5), (B,3,3) \rangle $\\
2&$\langle (B,3,3), (A,8,1), (B,9,3) \rangle$\\
3&$\langle (C,2,1), (C,5,2), (A,8,3) \rangle$\\ \bottomrule
\end{tabular}
\end{table}

The SPM task asks for the set of frequent \textit{patterns} within $\mathcal{SD}$. A pattern $P=\langle i_1, i_2, \dots, i_{|P|} \rangle$ is a subsequence of some $S \in \mathcal{SD}$. Let $S[j]$ denote the $j^{th}$ position (i.e., item) of sequence $S$. A subsequence relation $P \preceq S$ holds if and only if there exists an embedding $e: e_1 \le e_2 \le ... \le e_{|P|}$ such that $S[e_j]=i_j, i_j \in P$. For example, $P=\langle A, B \rangle $ is a subsequence of $S= \langle A, B, C, B \rangle$ with two possible embeddings $(1,2)$ or $(1,4)$. 
We define a super-sequence relation $S \succeq P$ analogously, with ``$\le$'' replaced by ``$\ge$''. A pattern is frequent if it is a subsequence of at least $\theta$ number of sequences in $\mathcal{SD}$, where $\theta$ is a given frequency threshold.

The two best-known mining algorithms for SPM are the Apriori algorithm introduced by \cite{agrawal1994fast}, and the prefix-projection algorithm introduced by \cite{han2001prefixspan}. Both are iterative procedures and operate by extending frequent patterns one item at a time. In Apriori, candidate patterns are generated by expanding a  pattern with all available items, and thereafter checking the frequency of generated candidates. As candidates may or may not be frequent, the Apriori algorithm suffers from the exponential explosion of the number of generated candidates and redundancy. The prefix-projection algorithm, in turn, operates by projecting 
each sequence $S \in\mathcal{SD}$ onto a smallest subsequence $\bar{S}=\langle i_1, i_2, \dots, i_j \rangle$, denoted by \textit{prefix}, and searching for frequent items in this reduced database. Any sequence that is obtained by 
extending a frequent prefix is guaranteed to be frequent in the original database. Prefix-projection is more efficient than the Apriori algorithm as it rules out infrequent patterns more effectively, but it requires the full database to be in memory \cite{han2001prefixspan}. 

\subsection{Constraint satisfaction in CSPM}
A constraint $C_{type}(\cdot)$ is a Boolean function imposed on either the patterns or their attributes. A pattern $P$ satisfies a constraint if and only if $C_{type}(P)=true$. The objective of CSPM is to find all frequent patterns that satisfy a set of user-defined constraints. In particular, the challenge of CSPM is to impose constraints during the mining algorithm, rather than post-processing mined patterns for constraint satisfaction. 

The standard framework for CSPM is to classify constraints as being monotone or
anti-monotone, as such constraint are easy to handle within the mining algorithm~\cite{pei2007constraint}.\footnote{A third classification is succinctness, which allows immediate pattern generation using a formula rather than an algorithm.}  
A constraint is {\em monotone} if its violation by a sequence $S$ implies that all subsequences $\bar{S}\preceq S$ also violate the constraint.  

A constraint is {\em anti-monotone} if its violation by a sequence $S$ implies violation by all super-sequences $\hat{S}\succeq S$.  
Table \ref{tab:conskind} lists common constraint types with their characterization.
The concepts of monotonicity, anti-monotonicity, and violation are analogously extended to prefixes.

\begin{table}
\caption{Characterization of constraints as monotone (M), anti-monotone (AM), or non-monotone (NM) for SPM.}
\label{tab:conskind}
\centering
\resizebox{\columnwidth}{!}{
\begin{tabular}{l l c c c}
	\toprule
	\textbf{Name}&\textbf{Constraint $:=$ definition}&\textbf{M}&\textbf{AM}&\textbf{NM}\\ \midrule
	Maximal&$C_{mxl}(P):= \nexists P' \in \mathcal{SD} : P  \prec P'  $&$\bullet$&&\\
	Sup-Patt&$C_{spt}(P):= \exists P' \in \mathcal{SD}: P' \prec P$ &$\bullet$&&\\ \midrule
	\multirow{2}{*}{Length}&$C_{len}(P)\ge c := |P| \ge c$&$\bullet$&&\\
	&$C_{len}(P)\le c$&&$\bullet$&\\\midrule
	Reg Expr&$C_{reg}(P):= P[i] \in \bar{I} \subset I$&&$\star$&\\\midrule
	\multirow{2}{*}{Gap}&$ C_{gap}(\mathcal{A})\le c := \alpha_{j} - \alpha_{j-1} \le c,$&&$\star$&\\
	&\hspace{2cm}$\alpha_{j} \in \mathcal{A}, 2 \le j \le |P|$\\
	&$C_{gap}(\mathcal{A})\ge c$&&$\bullet$&\\\midrule
	\multirow{2}{*}{Span}&$C_{spn}(\mathcal{A})\le c := \max\left\lbrace \mathcal{A}\right\rbrace -\min\left\lbrace \mathcal{A} \right\rbrace \le c $&&$\bullet$&\\
	&$C_{spn}(\mathcal{A})\ge c$&$\bullet$&&\\\midrule
	\multirow{2}{*}{Max/Min}&$C_{max}(\mathcal{A}) \ge c, C_{min}(\mathcal{A})\le c $&$\bullet$&&\\
	&$C_{max}(\mathcal{A}) \le c, C_{min}(\mathcal{A})\ge c $&&$\bullet$&\\\midrule
	Stats&$C_{sum}(\mathcal{A}),C_{avg}(\mathcal{A}), C_{var}(\mathcal{A}), C_{med}(\mathcal{A})$&&&$\bullet$\\\bottomrule
	\multicolumn{4}{l}{$^\star$Not anti-monotone, but prefix anti-monotone.}\\
	\end{tabular}}
\end{table}

Constraints that are neither monotone nor anti-monotone are called {\em non-monotone} and are the most challenging to enforce during mining. While dedicated approaches have been developed for certain non-monotone constraints~\cite{pei2007constraint}, they are otherwise handled by post-processing~\cite{aoga2017mining}.  Our goal is to develop a generic platform to handle non-monotone constraints effectively.  

\section{An MDD representation for $\mathcal{SD}$}\label{sec:MDDdef}

MDDs are widely applied as an efficient data structure in verification problems~\cite{wegener2000branching} and were more recently introduced as a tool for discrete optimization and constraint programming~\cite{BerCirHoeHoo2016}. Here, we use an MDD to fully encode the sequences from $\mathcal{SD}$; we refer to such data structure as an \textit{MDD database}. 
We show how constraint satisfaction is achieved by storing constraint-specific information at the MDD nodes, thereby removing the need to impose constraint-specific rules in a mining algorithm.

\subsection{MDD construction for the SPM problem}

An MDD $M=\left(U,A\right) $ is a layered directed acyclic graph, where $U$ is the set of nodes, and $A$ is the set of arcs. Set $U$ is partitioned into layers $(l_0, l_1, ..., l_{m+1})$, such that layers $l_i: 1\le i \le m$ correspond to position (item) $i$ of a sequence $S \in \mathcal{SD}$. Layers $l_0$, and $l_{m+1}$ consist of single nodes, namely the root node $r \in l_0$, and the terminal node $t \in l_{m+1}$. The root and terminal node are used to model the start and end of all sequences, respectively. Figure \ref{fig:mdd}.a shows the MDD database model for the $\mathcal{SD}$ of Table \ref{tab:data}.

Layers $l_j, 1\le j \le m$, contain one node per item $i \in I: \exists S \in \mathcal{SD} , S[j]=i$, and model the possible items at position $j$ of all sequences $S \in \mathcal{SD}$. For example, layer 1 of the MDD database in Figure \ref{fig:mdd}.a has two nodes corresponding to items $B, C$, and no node associated to item $A$. To distinguish which nodes are associated to which sequences $S\in \mathcal{SD}$, we define labels $d_u$ for nodes $u \in U$, and store the associated sequence index $S_{\rm ID}$ in $d_u$. The first label of node $B$ at layer 1 of Figure \ref{fig:mdd}.a, indicates that sequences 1 and 2 contain item $B$ at their first position.  In addition, we store the attribute labels associated with the item, one per $S_{\rm ID}$ at each node.  For example, in Figure~\ref{fig:mdd}.a we store the time and price attributes.

\begin{figure}
\centering
\includegraphics{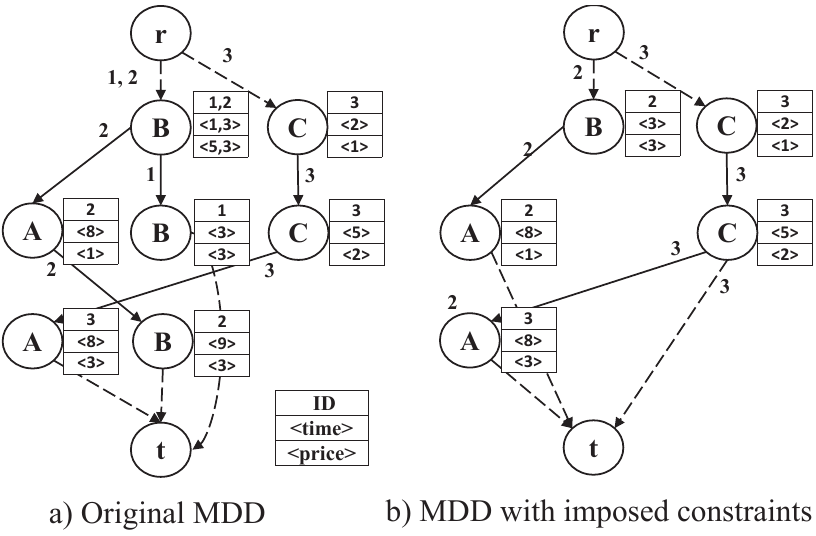}
\caption{MDD database for the example SD in Table~\ref{tab:data}. Arcs skipping layers in Figure a) are not shown for clarity.}
\label{fig:mdd}
\end{figure} 

An arc $a=(u,v) \in A$, is directed from a node $u \in l_j$ to a node $v \in l_{j'}: j'>j$, and represents the next possible item after node $u$, for all sequences in $\mathcal{SD}$. Similar to nodes $u \in U$, labels $d_a$ are defined for arcs $a\in A$ and store their associated sequences. A sequence $S$ is thus represented by a path from $r$ to $t$, following the nodes and arcs associated to $S_{\rm ID}$. 
As we will search the MDD for patterns during the mining algorithm, we explicitly allow arcs to skip layers.  That is, arc $(u,v) \in A$ can refer to any pair of nodes $u$, $v$ on an $r$-$t$ path $P$ representing a sequence $S$.  In Fig.~\ref{fig:mdd} we only depict the arcs that represent the original sequences in $\mathcal{SD}$, for clarity. For example, the arc between node $B$ at layer 1 and node $B$ at layer 3 (following sequence $S_{\rm ID}=2$) is formally defined but omitted from the picture. 
Observe that any prefix or subsequence is represented by a partial path in the MDD, possibly using the arcs that skip layers.
Lastly, we note that the MDD database (without imposed constraints) is built by a single scan of the database. 

\subsection{Imposing constraints on the MDD database}

We use the MDD structure to enforce certain constraints on the MDD database itself. This has 
three main benefits, as follows. First, constraint satisfaction is performed only once, and not once per projected database as in the prefix-projection algorithm. Second, several constraints can be considered simultaneously, as opposed to iterative methods that consider each constraint individually and incur larger computational costs. Lastly, imposing constraints results in a smaller MDD, and consequently reduced computational requirements for the mining algorithm.

A constraint $C_{type}$ can be imposed directly on the MDD if it is prefix monotone or prefix anti-monotone. That is, the feasibility of extending a pattern $P$ ending at item $i$ by an item $i'$, is only dependent on the relationship between consecutive items $i, i'$. 
Examples of such constraints are the gap and regular expression constraints. An infeasible extension of such constraints is prevented by not creating an arc between their respective nodes. For example, if item $i$ cannot be followed by item $i'$, then no arc of the MDD database is constructed between their corresponding nodes. 

Constraints on the MDD database are incorporated during its construction. In particular, the MDD database is built in increments using a backwards induction on the position $j$ of a sequences $S \in \mathcal{SD}$. A backwards induction is chosen, as it allows us to gather constraint-specific information, used for constraint satisfaction later in the mining algorithm. For sequence $S$, the algorithm starts from the node corresponding to the item at position $S[j]: j=|S|$, and checks whether this item may be used to extend a pattern ending in any of the sequence's previous items $i \in l_{j'}<l_j$. Whenever an extension is feasible, an arc $(u,v)$ is created between the items' respective nodes in the MDD. The algorithm then increments and repeats the same procedure for the item in position $j-1$. 

By the construction above, a node connects to all nodes representing a feasible extension with respect to the imposed constraints. Thus, the mining algorithm needs only to search the children of a node $u \in U$ to extend any pattern ending at $u$. Figure \ref{fig:mdd}.b shows an example of imposing constraint $C_{gap}(time)\ge 3$ on the MDD database of Figure \ref{fig:mdd}.a.

Imposing constraints on the MDD database can be made more efficient by exploiting their properties such as anti-monotonicity. For example, given an anti-monotone constraint, if the extension of item $i$ at $S[j]$ to an item at position $S[j']$ is infeasible, it is guaranteed that any extension of $i$ to items $S[k]: k\ge j'$ is also infeasible. If a constraint is non-monotone, we are required to check its satisfaction for all possible extensions, which is done only if all monotone and anti-monotone constraints are satisfied.

Not all constraints can be imposed on the MDD database. The satisfaction of such constraints is performed during the mining algorithm, discussed in the next section. 

\section{Pattern mining with MDD databases}\label{sec:effinfo}

In this section, we discuss how to perform constraint reasoning by incorporating 
specific information into the MDD nodes. Such information is used to establish conditions to efficiently
remove infeasible patterns from the database.

\subsection{Information exploitation for effective mining}

By construction, an $r\mdash u$ path in the MDD database represents the prefix of a pattern ending at node $u$. Similarly, any extension of this prefix is modeled by a $u \mdash t$ path. Post-processing patterns for constraint satisfaction corresponds to checking the feasibility of all $u\mdash t$ paths. 
We can, however, exploit the MDD structure to determine whether it is possible 
to extend an infeasible pattern to a feasible one. This is achieved by augmenting the MDD nodes with 
constraint-specific information that allow us to perform such reasoning.

For instance, consider a constraint $C_{min}(price)\ge 5$ and the extension of an infeasible pattern ending at node $u$, as shown in Figure \ref{fig:exmcons}. Observe that only one $u\mdash t$ path results in a feasible pattern. Instead of explicitly searching all $u\mdash t$ paths, we can store the minimum price reachable from nodes $u \in U$, during the MDD construction, and then use it to guarantee that a feasible extension exists. 

\begin{figure}
\centering
\includegraphics{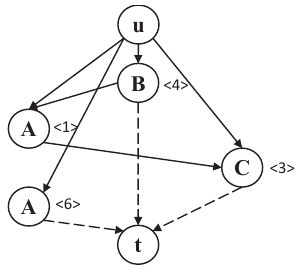}
\caption{Extending a pattern ending at node $u$, with constraint $C_{min}(\mathcal{A})\ge 5$. The label at each node represents the attribute of the item.}
\label{fig:exmcons}
\end{figure} 

\subsection{Categories of constraint-specific information}

We now describe constraint-specific information for a number of practical constraint classes.
We only present the proof for lower bound constraints; upper bound conditions can be established analogously. We define $\alpha^u \in \mathcal{A}$ to be the attribute value of item $i$ at node $u$ of the MDD.

\subsubsection{Span constraint:} 

Let $\beta^{u}_1$ and $\beta^{u}_2$ denote the minimum and maximum values of $\alpha$ reachable from $u$, respectively. Values $\beta^u_1$ are initially set to $\alpha^u$. When adding an arc $(u,v)$, we update $\beta^u_1 \leftarrow \beta^v_1$ if $\beta^u_1 > \beta^v_1$, and $\beta^u_2 \leftarrow \beta^v_2$ if $\beta^u_2 < \beta^v_2$ for node $v$. By this procedure, $\beta^u_1, \beta^u_2$ give the minimum and maximum values of $\alpha$ reachable from $u$. Proposition \ref{thm:spnnon} proves that by using these variables, we can guarantee the satisfaction of the span constraint.

\begin{proposition}\label{thm:spnnon}
An infeasible pattern $P$ can be extended to a feasible pattern with respect to $C_{spn}(\alpha) \ge c$ if and only if $\max\left\lbrace \underset{\alpha \in P}{\max}\left\lbrace \alpha\right\rbrace , \beta^u_2\right\rbrace  - \min\left\lbrace \underset{\alpha \in P}{\min}\left\lbrace \alpha \right\rbrace , \beta^{u}_1\right\rbrace  \ge c$.
\end{proposition}
\begin{proof}
The necessity is straightforward. For the converse, assume $\alpha^{\max} - \alpha^{\min} < c$. Then no $u\mdash t$ path contains values of $\alpha$ such that $P$ can become feasible. 
\end{proof}

\subsubsection{Sum constraint:} 
Let $\beta^u$ denote the maximum sum of values $\alpha$ reachable from $u$. We first initialize $\beta^u \leftarrow \alpha^u$. Next, when adding an arc $(u,v)$, we update $\beta^u \leftarrow \beta^v + \alpha^u$ if $\beta^u < \beta^v + \alpha^u$, which results in the maximum sum possible to be stored for node $u$. Proposition \ref{thm:sum} proves that this information is sufficient.  

\begin{proposition}\label{thm:sum}
There exists a feasible extension from node $u$ with respect to each individual constraint if and only if $\sum\limits_{\alpha \in P}\alpha + \beta^{u} \ge c$.
\end{proposition}
\begin{proof}
The necessity is straightforward. For the converse, assume $\sum\limits_{\alpha \in P}\alpha + \beta^{u} < c$. By the construction of $\beta^u$, we can conclude  $\sum\limits_{\alpha \in P}\alpha + \sum\limits_{\alpha \in (u,t)}\alpha < c$, for all $u\mdash$ paths.
\end{proof}

\subsubsection{Average constraint:} 

Let $\beta^u_1$ denote a sum of values $\alpha$ on a $u \mdash t$ path, and $\beta^u_2$ denote the number of attributes $\alpha$ contributing to the sum in $\beta^u_1$. For constraint $C_{avg}(\alpha) \ge c$, and any pattern $P$ ending at node $u$, our objective is to generate values of $\beta^u_1, \beta^u_2$ that give the maximum possible average $\frac{\sum\limits_{\alpha \in P}\alpha + \beta^u_1}{|P| + \beta^u_2}$ above the threshold $c$. 

The generation of $\beta^u_1$ depends on the value $c$ of constraint $C_{avg}(\alpha) \ge c$. Initially we set $\beta^u_1$ = $\alpha^u$, and $\beta^u_2 = 1$. When adding an arc $(u,v)$ during the construction of the MDD, we update $\beta^u_1 \leftarrow \alpha^u + \beta^v_1, \beta^u_2 \leftarrow \beta^v_2 + 1$ if $\left( \alpha^u + \beta^v_1\right) - c\left( 1 + \beta^v_2\right) > \beta_1^u - c\beta^u_2$. This ensures that the best values to maximize $\frac{\sum\limits_{\alpha \in P}\alpha + \beta^u_1}{|P| + \beta^u_2}$ are generated, proven in Lemma \ref{lem:avr}.

\begin{lem}\label{lem:avr}
For constraint $C_{avg}(\alpha) \ge c$, the update procedure above generates values $\beta^u_1, \beta^u_2$ that give the maximum average $\frac{\sum\limits_{\alpha \in P}\alpha + \beta^u_1}{|P| + \beta^u_2}$ above threshold $c$, for a pattern $p$ ending at node $u$.
\end{lem}
\begin{proof}
Proof by induction. By the initial definitions of $\beta^u_1, \beta^u_2$, the statement is true for any $u$ in the last layer $l_{m}$ of the MDD. Now assume the statement holds for all nodes in layer greater than $l_{j}$. For nodes $u$ in layer $l_{j}$ we choose the path giving the maximum average $\underset{u\mdash t}{\max}\left\lbrace \frac{\sum\limits_{\alpha \in P}\alpha + \beta^v_1 + \alpha^u}{|P| + \beta^v_2 + 1} - c\right\rbrace = \underset{u\mdash t}{\max}\left\lbrace \beta^v_1 + \alpha - c\left(\beta^v_2 + 1 \right)\right\rbrace$.    
\end{proof}

Proposition \ref{thm:avg} shows that $\beta_i^1, \beta_u^2$ are the only required information to check satisfaction of the maximum average constraint. The proof for the minimum average constraint is similar, and omitted for brevity.  

\begin{proposition}\label{thm:avg}
It suffices to record $\beta_1^u, \beta_2^u$ as defined above, to check satisfaction for the minimum average constraint $C_{avg}(\alpha) \ge c$.
\end{proposition}
\begin{proof}
The maximum average reachable from node $u$ is $\frac{\beta^u_1}{\beta^u_2}$ by definition. Therefore, if $\frac{\sum\limits_{\alpha \in P}\alpha + \beta^u_1}{|P| + \beta^u_2} < c$, then no $(u,t)$ paths exists that satisfies $C_{avg}(\alpha) \ge c$ for a pattern ending at node $u$.
\end{proof}

\subsubsection{Median constraint:} Let the maximum difference of the number of values $\alpha \ge c$ and the number of values $\alpha < c$, between all possible paths $u\mdash t$, i.e. $\beta^u_1 = \underset{u\mdash t}{\max}\left\lbrace \bigm|\left\lbrace \alpha \in u\mdash t : \alpha \ge c \right\rbrace \bigm| - \bigm|\left\lbrace \alpha \in u\mdash t: \alpha < c \right\rbrace \bigm| \right\rbrace$. Further, let $\beta^u_2$ denote the maximum of values $\alpha < c$ contributing to the count in $\beta^u_1$, and $\beta^u_3$ denote the minimum of values $\alpha \ge c$ contributing to the count in $\beta^u_1$. Observe that the satisfaction of $C_{med}(\alpha) \ge c$ can be determined using values $\beta^u_1-\beta^u_3$. Namely, if $\beta^u_1 > 0$ then there exists more values $\alpha$ above $c$ than below it, guaranteeing satisfaction. Similarly if $\beta^u_1 < 0$ the median constraint is violated. If $\beta^u_1 = 0$ then we calculate the average $\frac{\beta^u_2 + \beta^u_3}{2}$ which gives the median.

The generation of $\beta^u_1$ to $\beta^u_3$ depends on the constant $c$. Initially, we set $\beta_1^u = 0, \beta_2^u = \underset{\alpha \in S}{\min}\left\lbrace \alpha\right\rbrace - 1 , \beta_3^u = \alpha^u$ for all nodes $u: \alpha^u \ge c$, and $\beta^u_1 = 0, \beta_2^u = \alpha^u, \beta^u_3 = \underset{\alpha \in S}{\max}\left\lbrace \alpha\right\rbrace + 1$ for all remaining nodes. Next, during the construction of the MDD, for a node $u$, we find the path $u\mdash t$ that has the highest potential to extend an infeasible pattern $P$ ending at $u$ to a feasible one. The best path $u\mdash v\mdash t$, denote $v\mdash t$, is a path that contains a feasible extension for $P$ given any other feasible extensions available by the remaining $u\mdash v'\mdash t$ paths, denote $v'\mdash t$. We prove four dominance rules that when satisfied, guarantee this for $v\mdash t$.

The first rule is if $\beta_1^v > \beta_1^{v'}$, proven valid in Lemma \ref{lem:med1}. 
\begin{lem}\label{lem:med1}
If $\beta_1^v > \beta_1^{v'}$ holds, and extension of a pattern $P$ by path $v'\mdash t$ is feasible, so is the extension of $P$ by path $v \mdash t$. 
\end{lem}
\begin{proof}
Let $\beta_1^p$ denote the difference of the number of values $\alpha \in P: \alpha \ge c$ to the number of values $\alpha \in P: \alpha < c$. Then, $\beta_1^p + \beta_1^v > \beta_1^p + \beta^{v'}_1$, meaning there is a greater number of values $\alpha \ge c$ on path $v\mdash t$, compared to path $v' \mdash t$.
\end{proof}
All other conditions require $\beta_1^v = \beta_1^{v'}$. For these conditions, we first calculate $med_{v'} = \frac{\beta_2^{v'} + \beta_3^{v'}}{2}, med_v = \frac{\beta_2^v + \min\left\lbrace \beta_3^v, \alpha^v\right\rbrace}{2}$. Conditions two to four are proved in \ref{lem:med2}.
\begin{lem}\label{lem:med2}
Given $\beta_1^v = \beta_1^{v'}$, any feasible extension of an infeasible pattern $P$ by path $v' \mdash t$ is also feasible for path $v\mdash t$, if one of the following three conditions hold: 1. $med_v \ge c, med_{v'} < c$, 2. $med_v \ge c, med_{v'} \ge c, \beta_2^v > \beta_2^{v'}$, 3. $med_v < c, med_{v'} < c, \beta_3^v > \beta_3^{v'}$.
\end{lem}
\begin{proof}
Let $\beta^p_1$ to $\beta^p_3$ be defined as before. For condition 1, as $med_{v'} < c$ and $P$ is infeasible, any extension of $P$ by $u\mdash t$ must have $\beta_1^p + \beta_1^{v'} > 0$, which is also satisfied by path $v \mdash t$. For condition 2, if an infeasible pattern $P$ can be extended to a feasible pattern by $v' \mdash t$, then either $\beta_1^p + \beta_1^{v'} > 0$ which implies feasibility of $v \mdash t$, or $\beta_1^p + \beta_1^{v'} = 0$. In this case, the only value of $\frac{\max\left\lbrace \beta_2^p, \beta_2^{v'}\right\rbrace + \min\left\lbrace \beta_3^p, \beta_3^{v'}\right\rbrace}{2}$ (i.e., the median of pattern $P$ extended by $v'\mdash t$), which is not guaranteed to be feasible or infeaisble is $\frac{ \beta_2^{v'} + \beta_3^p}{2}$. However, if $\frac{ \beta_2^{v'} + \beta_3^p}{2} \ge c$, we also have $\frac{ \beta_2^{v} + \beta_3^p}{2} \ge c$. The proof of the third rule is similar to the second rule, and ommited due to space limits.
\end{proof}

If any of the above rules are satisfied, we update $\beta_1^u \leftarrow \beta_1^v + 1, \beta_2^u \leftarrow \max\left\lbrace \beta_2^v, \alpha \right\rbrace, \beta_3^u \leftarrow \beta_3^v$ if $\alpha^u \ge c$, or $\beta_1^u \leftarrow \beta_1^v - 1, \beta_2^u \leftarrow \beta_2^v, \beta_3^u \leftarrow \max\left\lbrace \beta_3^u, \beta_3^v \right\rbrace$ otherwise. Proposition \ref{thm:med} shows that these values are sufficient to determine whether an infeasible pattern $P$ can be extended to a feasible one.

\begin{proposition}\label{thm:med}
Let $\beta^p_1-\beta^p_3$ be defined as before. There exists a feasible extension from node $u$ with respect to $C_{med}(\alpha) \ge c$ if and only if $\beta_1^u + \beta_1^p > 0$, or $\beta_1^u + \beta_1^p = 0, \frac{\min\left\lbrace \beta_3^p, \beta_3^p\right\rbrace + \max\left\lbrace \beta_2^p, \beta_2^u\right\rbrace}{2} \ge c$. 
\end{proposition}
\begin{proof}
The necessity is straightforward. For the converse, first assume $\beta_1^p + \beta_1^v < 0$, then by Lemma \ref{lem:med2}, no $u\mdash t$ path contains enough values $\alpha \ge c$ to satisfy $C_{med}(\alpha) \ge c$. For the second condition, if $\beta_1^u + \beta_1^p = 0, \frac{\min\left\lbrace \beta_3^p, \beta_3^u\right\rbrace + \max\left\lbrace \beta_2^p, \beta_2^u\right\rbrace}{2} < c$, then by Lemma \ref{lem:med2}, the maximum median between all $u\mdash t$ paths is below threshold $c$.
\end{proof}

\section{Mining the MDD database with prefix-projection}\label{sec:MDDmining}

We now present our \textit{MDD prefix-projection} (MPP) algorithm, which performs prefix-projection on the MDD database. The first step of the algorithm is to find all frequent items $i$, i.e. patterns of size one, using a depth-first-search.  This is automatically done during the construction of the MDD database, and modeled by the children of root node $r$.  In the next steps, the algorithm attempts to expand a frequent pattern generated in previous iterations. Using the stored information in the MDD, we prune extensions that cannot lead to a feasible pattern. In particular, for an infeasible pattern $P$ ending at node $u \in U$, the algorithm uses the information stored at $u$ to determine whether $P$ may be extended to a feasible pattern. If a feasible extension does not exist, the search is pruned. Otherwise, pattern $P$ is extended and investigated in future iterations. 


In contrast to searching the database rows in prefix-projection, the MPP algorithm follows feasible paths in the MDD database. This leads to a more efficient search, as some infeasible extensions have been removed when constructing the MDD database. The trade-off is that finding paths corresponding to a sequence $S$ requires a search on labels $d_u, a_u$, thereby incurring additional computational cost. For efficient memory utilization, the MDD is not physically projected, but rather \textit{pseudo projected} \cite{han2001prefixspan}. In pseudo projection, only the initial $\mathcal{SD}$ is stored in memory, and search is initiated from ``projection pointers'' pointing to the MDD nodes.

In prefix-projection, all $N$ sequences are searched in each iteration, and an item $i \in I$ is frequent if its final count is at least $\theta$. As opposed to searching all $N$ sequences, we propose to stop when it is guaranteed that an item $i$ is not frequent. Let $n$ denote the number of sequences searched so far when searching for $i$, and let $Sup(P)$ denote the number of sequences that contain pattern $P$.  We use the following proposition to detect that item $i$ cannot be frequent, given a frequent pattern $P$:

\begin{proposition}\label{thm:theta}
If $n - Sup(i) > Sup(P) - \theta$, item $i$ cannot be frequent in the projected database.
\end{proposition}
\begin{proof}
The left-hand-side is the number of searched sequences that do not contain $i$, and the right-hand-side is the maximum number of sequences that do not contain $i$ while it remains frequent. 
\end{proof}

Projecting the minimal prefix containing a pattern $P$ (as done in SPM) is not sufficient for CSPM \cite{aoga2017mining}. Extensions from the minimal prefix may violate a constraint, while it may be the case that another larger prefix of the sequence satisfies such extensions. For example, the minimal prefix containing item $C$ in sequence 3 of Table \ref{tab:data} cannot be extended by item $A$ under a constraint $C_{gap}(time)\le 3$. However, extending the larger prefix containing $C$ is feasible. We are thus required to store all prefixes and their extension at each iteration of MPP. 

A time-consuming task of the general prefix-projection algorithm is to determine whether a specific item $i$ exists in sequences of the projected database. To avoid searching the entire sequence for every item, \cite{aoga2017mining} store the last position of items $i \in I$ for sequences $S \in \mathcal{SD}$. An MDD database enables search for the extension of all items $i \in I$ simultaneously, resulting in more efficient search. That is, as opposed to searching for a specific item $i$, all children of node $u$ are searched, and record the items which enable a feasible extension.


\section{Numerical results}\label{sec:results}

\begin{table}[t]
\centering
\caption{Five real-life datasets and their features.}
\label{tab:inst}
\resizebox{\columnwidth}{!}{
	\begin{tabular}{l l l l l}\toprule
	$\mathcal{SD}$& $N$ & $|I|$ & $M $& $avg(|S|)$*\\\midrule
	Kosarak&837,206&41,001&2,498&9.3\\
	MSNBC&989,818&19&29,591&10.5\\
	Kosarak (small)&59,261&20,894&796&9.2\\
	BMSWebView1&26,667&497&267&4.4\\
	BMSWebView2&52,619&3,335&161&6.3\\\bottomrule
	\multicolumn{5}{l}{*Average length of sequences}
	\end{tabular}}
\end{table}

For our numerical tests, we use real-life click-stream benchmark databases\footnote{http://www.philippe-fournier-viger.com/spmf/index.php?link=datasets.php}, listed in Table~\ref{tab:inst}.  We note that two of these databases, Kosarak and MSNBC, are considerably larger than those typically reported in the CSPM literature, with about 900,000 sequences of length up to 29,500, and containing up to 40,000 items. None of these standard benchmark datasets are annotated with attributes.  To be able to evaluate our approach, we therefore generate three attributes of time, price, and quality, as follows. For the time attribute, we randomly generate a number between 0 and 600 seconds, to represent the time spent by users at each click. With a probability of 5\%, we model the user leaving the session by setting the time between clicks to a value between 1 to 10 hours. For the price and quality attributes, we generate a number between 1 and 100 for each item $i \in S, \forall S \in \mathcal{SD}$.

All algorithms are coded in C++, with the exception of PPICt which is coded in Scala.\footnote{We thank the developers of PPICt for sharing their code.}  All experiments are executed on the same PC with an Intel Xeon 2.33 GHz processor, 24GB of memory, using Ubuntu 12.04.5 as operating system.  We limit all tests to use one core of the CPU.  The MPP code is available and open source.\footnote{https://github.com/aminhn/MPP}

\subsection{Comparison with prefix-projection and constraint checks}

Our first goal is to evaluate the impact of the MDD database and the associated constraint reasoning, especially in presence of more complex constraints.  However, no other CSPM system accommodates constraints such as average and median and multiple item attributes.  Because simple generate-and-test (via post-processing) does not scale due to the size of the databases, we developed a prefix-projection algorithm for the original database, that can handle multiple item attributes and effectively prune the search space for anti-monotone constraints such as gap and maximum span.  
We name this algorithm Prefix-Projection with Constraint Checks (PPCC). PPCC operates by prefix-projection and extends a pattern $P$ if it satisfies all anti-monotone constraints, and prunes the extension otherwise. For non-monotone constraints, PPCC extends infeasible patterns with the hope that a feasible super-pattern exists, and performs a constraint check at the end.

In Figure~\ref{fig:ppcc} we compare the performance of MPP and PPCC in terms of total CPU time (MDD construction plus mining algorithm), given minimum support (Min supp) as a percentage of the total number of sequences.  The experiment uses three scenarios with constraints on one, two, and three attributes, respectively:
\begin{small}
\[
\begin{array}{l}
\textrm{time:~} 30 \leq C_{gap}(time) \leq 900,~ 900 \le C_{spn}(time) \le 3600, \\
\textrm{price:~} 30 \le C_{avg}(price) \le 70,~  40 \le C_{med}(price) \le 60, \\
\textrm{quality:~} 40 \le C_{avg}(quality) \le 60,~ 30 \le C_{med}(quality) \le 70.
\end{array}
\]
\end{small}

Scenario one (PPCC1 and MPP1) only considers the time constraints.  Scenario two (PPCC2 and MPP2)  considers the time and price constraints.  Scenario three (PPCC3 and MPP3) considers all time, price, and quality constraints.
The results in Figure~\ref{fig:ppcc} show that mining more constrained patterns takes more time for both methods.  However, MPP is always more efficient than PPCC, and often considerably.  For example, finding all frequent patterns with minimum support of 4\% with all constraints (scenario three) in the MSNBC database takes PPCC about 4,000s while MPP only needs about 2,000s. Moreover, Table \ref{tab:mddtime} shows that the time required to construct the MDD database and generate constraint specific information is quite small. This indicates that our MDD database can be used to effectively and efficiently handle constraints such as average and median.

\begin{figure}	
\centering
\includegraphics{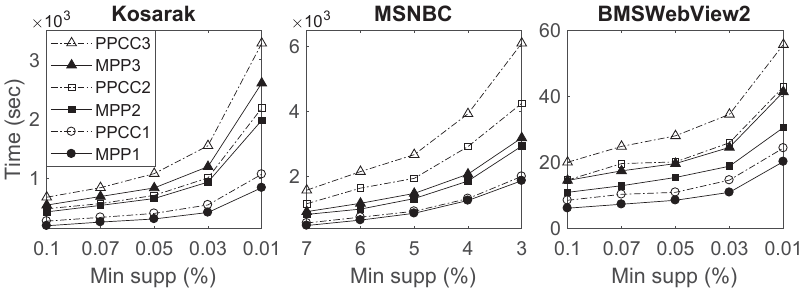}
\caption{Mining with constraints $30 \le C_{gap}(time) \le 900, 900 \le C_{spn}(time) \le 3600, 30 \le C_{avg}(price) \le 70, 40 \le C_{med}(price) \le 60, 40 \le C_{avg}({\it quality}) \le 60, 30 \le C_{med}({\it quality}) \le 70$. Attributes and their corresponding constraints are added incrementally from 1 to 3.}
\label{fig:ppcc}
\end{figure}

\subsection{Comparison with PPICt}
We next compare our approach to the state-of-the-art CSPM algorithm PPICt, which is implemented in the CP framework OscaR\footnote{https://bitbucket.org/oscarlib/oscar/wiki/Home}~\cite{aoga2017mining}.  PPICt accommodates a wide range of constraints, including gap and maximum span constraints, but is restricted to a single attribute.  We therefore evaluate MPP and PPICt for mining patterns with the following gap and maximum span constraints over the time attribute:
\begin{small}
\[
30 \leq C_{gap}(time) \leq 90,~~900 \le C_{spn}(time) \le 3600.
\]
\end{small}
\noindent Initial tests indicated that the PPICt code is unstable when executed on the full databases Kosarak and MSNBC.  We therefore executed the codes on the smaller benchmark variant of Kosarak (which is also used in \cite{aoga2017mining}), BMSWebView1, and BMSWebView2.  The results are presented in Figure~\ref{fig:ppict}, which follows the same format as Figure~\ref{fig:ppcc}.

\begin{figure}
\centering
\includegraphics{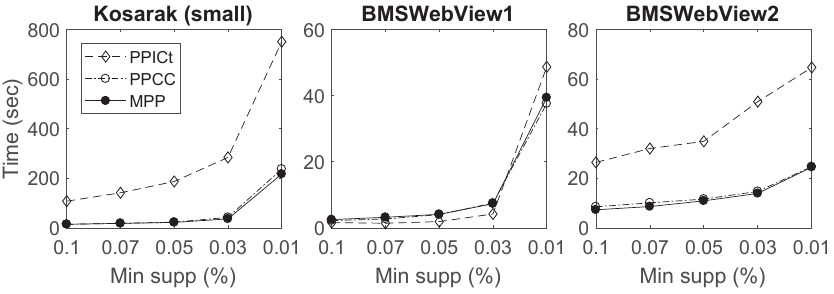}
\caption{Mining with one item attribute (time) and constraints $30 \le C_{gap}(time) \le 900, C_{spn}(time) \le 3600$.}
\label{fig:ppict}
\end{figure}  

A first observation is that MPP and PPCC produce almost identical results, as they both benefit from the same pruning rules for anti-monotone constraints. The time required to build the MDD database, shown in Table \ref{tab:mddtime}, is made up by a faster prefix-projection algorithm due to implementing the gap constraints on the MDD itself.  Both MPP and PPCC also outperform PPICt on Kosarak (small) and BMSWebview2, but all three methods perform similarly on BMSWebView1.  However, PPICt uses significantly more memory, up to 14Gb, while MPP uses up to 1Gb, and PPCC consumes the lowest with at most 0.5Gb.  We conclude that on this benchmark our approach is competitive with or more efficient than PPICt.

\begin{table} 
\centering
\caption{Time (in seconds) required for MDD construction and information generation.}
\label{tab:mddtime}
\resizebox{\columnwidth}{!}{
	\begin{tabular}{c c c c c c}\toprule
	Algorithm & Kosarak & MSNBC & BMS2 & BMS1 & Kosarak(small)\\ \midrule
	MPP1& 47 & 45 & 2 & - & -\\
	MPP2& 106 & 103 & 4 & - & -\\
	MPP3& 151 & 158 & 5 & - & -\\
	MPP& - & - & 2 & 0.5 & 4\\\bottomrule
	\end{tabular}}
\end{table}

\section{Conclusion}\label{sec:concl}

In this paper, we developed a novel MDD representation for CSPM. We prove how constraint satisfaction is achieved for a number of constraints, including sum, average, and median, by storing constraint-specific information at the MDD nodes. Moreover, our approach is able to accommodate several item attributes with constraints, which occur frequently in real-world problems. 

We embedded our MDD representation within a prefix-projection algorithm, called MPP, and performed an experimental evaluation on real-life benchmark databases with up to 980,000 sequences and 40,000 items.  The results showed that the MPP mining algorithm is always more efficient than a prefix-projection algorithm with constraint checks. The benefits of MPP become larger as we increase the size of the database, the number of constraints, or the number of attributes. Although MPP is primarily designed for efficient constraint satisfaction of rich constraints and multiple item attributes, it remains competitive with a CP-based state-of-the-art CSPM algorithm, for databases with only one item attribute and anti-monotone constraints.

\bibliographystyle{aaai}
\bibliography{PSMlib}

\end{document}